\documentclass[submission,copyright,creativecommons]{eptcs}
\usepackage{underscore}           
\usepackage{tikz}  
\usepackage{url} 
\usepackage{mathtools}    

\usepackage{enumitem,multicol,setspace}

\DeclarePairedDelimiterX{\infdivx}[2]{(}{)}{%
  #1\;\delimsize\|\;#2%
}

\DeclarePairedDelimiter{\norm}{\lVert}{\rVert}

\usetikzlibrary{matrix,arrows.meta,fit,decorations.markings,intersections,arrows,positioning,shapes.misc}  
\usepackage{tcolorbox} 
\tcbuselibrary{raster,skins}
\usepackage{rotating}    
    
\usepackage{tikz-cd}  
 
\newcommand{\inL}{\textsf{in}}
\newcommand{\outL}{\textsf{out}}
\newcommand{\undecL}{\textsf{und}}
\newcommand{\domit}{\mathsfit{dom}}  

\newcommand{\andC}{\textsf{and}} 

\newcommand{\FFMMA}{\ensuremath{(A, R)}} 
\newcommand{\Table}{\textsf{tbl}}   
\newcommand{\TABLE}{\textsf{TBL}}
\newcommand{\Tableit}{\mathsfit{tbl}}
\newcommand{\TABLEit}{\mathsfit{TBL}}

\newcommand{\hide}[1]{}
\newcommand{\header}{\mathsfit{header}}
 
\newcommand{\countit}{\textsf{count}} 
\newcommand{\SCHEMAit}{\mathsfit{SCHM}} 
\newcommand{\Schema}{\textsf{schm}}
\newcommand{\Schemait}{\mathsfit{schm}}
  
\newcommand{\evalit}{\mathsfit{eval}}

\newcommand{\sem}{sem}

\usepackage{pifont} 
\usepackage{amssymb}            
\usepackage{amsmath} 
\usepackage{bm}
\DeclareMathAlphabet{\mathsfit}{T1}{\sfdefault}{\mddefault}{\sldefault}
\SetMathAlphabet{\mathsfit}{bold}{T1}{\sfdefault}{\bfdefault}{\sldefault}

\newtheorem{definition}{Definition} 
\newtheorem{example}{Example}  
 
\newtheorem{proposition}{Proposition}  
\newtheorem{proof}{Proof} 

\title{Relational Argumentation Semantics} 
\author{
Ryuta Arisaka \qquad\qquad Takayuki Ito
\institute{Kyoto University\\
Kyoto, Japan}
\email{\quad ryutaarisaka@gmail.com \quad\qquad ito@i.kyoto-u.ac.jp}
}

\begin{document}
\maketitle

\begin{abstract}  
    In this paper, we propose 
	a fresh perspective on argumentation semantics, 
	to view them as a relational database. 
	It
	offers encapsulation of the underlying 
	argumentation graph, and allows us to  
	understand argumentation semantics  
	under a single, relational perspective, leading 
	to the concept of relational argumentation semantics. 
	This is a direction to understand argumentation semantics 
	through a common formal language. 
	We show that many existing semantics such as 
	explanation semantics, multi-agent semantics, 
	and more typical semantics, 
	that have been proposed for specific purposes,  
	are understood in the relational perspective. 
\end{abstract}

\section{Introduction}    
Relational database promotes knowledge management. Formal 
argumentation \cite{Dung95} (more traditionally belief revision \cite{Makinson85,Katsuno92})  
offers a machinery for conflict resolution. 
It is  natural to try to integrate the two.

The relation of database and formal 
argumentation has been 
primarily for the latter to augment the former for 
dealing with
inconsistent information present in it. There are a number of ideas proposed in this direction: 
inconsistent database management, e.g. \cite{Amgoud07,Arioua17,Deagustini13,Jefferys06,Hunter01,Hao19,Kokciyan18,Moguillansky12,Maria14,Leopoldo05,Novak15,Santos10}), 
argumentation or dialectical databases \cite{Capobianco05,Pradhan03}, 
argumentation-supported database queries 
\cite{Muller14,Gordon08}, 
schema matching \cite{Nguyen13}, argumentation 
semantic web \cite{rahwan07}.  

By contrast, the other direction of mingling, whereby 
formal argumentation is primary and relational database theory 
acts on it, has not been as well explored.
The contrast
could be in principle  due to a lack of practical or theoretical 
significance in 
such matrimony. We shall demonstrate, however, that 
that is far from the case; that, in fact, 
it gives rise to a new means of encapsulation 
to facilitate a more general understanding of argumentation semantics; 
uniting, 
in particular, disparate argumentation semantics 
proposed for different purposes under a single relational perspective.

To explain what we specifically propose, a broad 
description of formal argumentation should be in place first and 
foremost. Given a representation of argumentation  
as a directed graph of: nodes representing information (typically 
arguments though can be anything) and edges representing certain directional
relation (of attacks typically), 
formal argumentation typically identifies which nodes 
are accepted (, rejected, undecided-to-be-either, and so on). 
In a labelling-based approach, it amounts to 
finding labelling functions that assign 
a label corresponding to those 
acceptance statuses to each graph node in a way 
satisfying a given set of labelling constraints. Depending on 
chosen constraints, there could be more 
than one labelling function meeting the requirement, and there could be none. 

On to our proposal, let us say that 
there is an argumentation graph and a set of constraints, 
and that we have obtained all labelling functions as satisfy 
them. Our proposal is {\it to view them 
as a relational database}, to be manipulated upon, 
just as any relational database. 
Encapsulation of the underlying graph structure 
is one of the key advantages of this viewpoint, 
allowing us to understand 
argumentation semantics more generally through the abstraction. 
We show that once some primary semantics 
are obtained and put in a database table form, many other 
semantics are its derivatives in the relational perspective.   

We formalise our proposal by formulating
a tuple relational calculus (see any standard 
text, e.g. \cite{Ramakrishnan02}, for database, but 
see \cite{Kolaitis09} for an overview of theoretical relational 
database) 
for argumentation semantics. We then show that 
it serves to bring  
different argumentation semantics 
into a single relational perspective. 
The novelty, significance and relevance of this work, 
from our vantage point, are as follows.  
	\textbf{Novelty}: as far as our search on Google Scholar 
		has permitted us, we did not find 
		any previous work with a similar intent 
		to understand argumentation semantics 
		in tuple relational calculus. 
		The perspective to view 
		argumentation semantics as 
		a relational database, and 
		any query result 
		as a relational argumentation semantics, 
		is novel. 
	\textbf{Significance}: as we show, the relational perspective 
		 brings  
	 	argumentation semantics 
		proposed for different purposes into a single relational 
		perspective. 
		The abstraction helps 
		decouple primary components from derivatives, 
		allowing us to understand argumentation semantics 
		more generally. An argumentation semantics 
		has been proposed 
		for specific sets of constraints so far. 
		Our proposal is a way of classifying 
		argumentation semantics through a common formal 
		language, in that sense assisting 
		a more uniform understanding of argumentation semantics. 
		Furthermore, it is practically significant, 
		since any sensible query in the tuple relational 
		calculus that we propose here can be directly 
		handled in SQL; there is no need for any extra 
		implementation. 
	\textbf{Relevance}:  the topic is 
		of interest not only to 
		 argumentation communities and 
		the theoretical database community, 
		but also to those communities working 
		on extraction and management of knowledge, where 
		encapsulation is an asset.

\section{Technical Preliminaries} 
Tuple relational calculus is the theoretical foundation of SQL. 
However, it is not an entirely familiar concept 
to formal argumentation practitioners taken as a whole. 
What may be obvious to the database community  
needs to be introduced with sufficient detail. To the database community, 
tuple relational calculus - even though to be 
adjusted for the benefit of argumentation semantics - is indeed a 
fairly standard 
formal machinery needing barely any new introduction.  
On the other hand, what of argumentation semantics, 
and first of all what exactly an argumentation semantics is, 
may need to be formally detailed before it becomes clear what 
contributions 
the relational calculus can make to formal argumentation. 

These points weighed in, 
we choose to introduce 
formal argumentation semantics in this preliminary section 
in a way accessible to non-formal-argumentation communities, 
and treat adaptation of tuple relational calculus 
in Section 3 in a way accessible to the formal argumentation community. 

\subsection{Formal argumentation}  
Assume $\mathcal{A}$ is the class of 
abstract entities we understand as arguments, 
and $\mathcal{R}$ is the class of all binary relations 
over $\mathcal{A}$. We refer to a member of $\mathcal{A}$ by $A$, 
that of $\mathcal{R}$ by $R$, with or without a subscript. 
Assume $\mathcal{R}^{A}$ denotes the subclass of $\mathcal{R}$ 
containing every - but no other - $R \in \mathcal{R}$ 
that satisfies: if $(a_1, a_2) \in R$, then 
$a_1, a_2 \in A$. Then, a graph structure 
$(A, R)$ with $R \in \mathcal{R}^A$ is an abstract representation 
of an argumentation. Far more often than not, 
$A$ is finite, and we treat it as such in the rest. As for 
$R$, 
we assume for simplicity that every member indicates an attack relation, that is, 
$(a_1, a_2) \in R$ means $a_1$ attacks $a_2$, and $a_1$ is the attacker 
of $a_2$. 
\begin{example}[Warring countries, an example of an argumentation graph] 
    Suppose 7 warring countries, each having its own agenda as regards  
    which countries to intrude into:   
	{\small 
	\begin{multicols}{2}
	\begin{description} 
		\item[$a_A$:] Country A ``We destroy Country C.'' 
		\item[$a_B$:] Country B ``We destroy Country A.'' 
		\item[$a_C$:] Country C ``We destroy Country B.''
		\item[$a_D$:] Country D ``We destroy Country C and Country E.''
		\item[$a_E$:] Country E ``We destroy Country D.''  
		\item[$a_F$:] Country F ``We destroy  Country G.'' 
	\end{description}
	\end{multicols}   
	}
	Each country's stance, as an argument, and the direction of aggression  
generate the following argumentation graph: $(\{a_A, 
	\ldots, a_G\}, \{(a_F, a_G), (a_E, a_D), (a_D, a_E), 
	(a_D, a_C), (a_C, a_B), (a_B, a_A), (a_A, a_C)\})$. \hfill$\clubsuit$

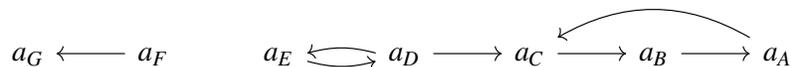
	\begin{center} 
	\begin{tikzcd} 
		a_G & a_F \arrow[l] & a_E \arrow[r,bend right=15] & a_D \arrow[l, shift left, bend right=15] \arrow[r] 
		& a_C \arrow[r] & a_B \arrow[r] & a_A \arrow[ll, bend right=30] 
	\end{tikzcd} 
	\end{center}  
\end{example}

Assume $L$ to be a set of labels, and assume $\Lambda$ to be the class of all
partial functions $\mathcal{A} \rightarrow L$. 
Each member of $\Lambda$ 
is a labelling function, though, 
as per a permeated convention, 
we simply call it a labelling in the rest. 
We refer to each labelling by $\lambda$ 
with(out) a subscript. 

One of the main objectives of the abstract representation  
of an argumentation is to 
infer from it acceptance statuses 
of arguments, to be determined by a given set of constraints.  
In an equivalent term, 
it is to 
derive a set of labellings that satisfy labelling constraints. 
As for what constraints exist, there are many that have been - 
and are still being - proposed 
for specific purposes.   
Assume $\domit(\lambda)$ is the domain of $\lambda$, and assume 
a label $\inL \in L$ indicating `accept'. 
Given $(A, R)$, there is the constraint of `conflict-freeness': 
$\lambda \in \Lambda$ be such that  
$\domit(\lambda) \subseteq A$
and that, for any $a_1, a_2 \in \domit(\lambda)$, 
if $\lambda(a_1) = \lambda(a_2) = \inL$, 
then $(a_1, a_2) \not\in R$. 
This condition enforces no simultaneous 
acceptance of an argument and its attacker (no
destroyed country can destory another country). 
The constraint of `$A_1$ 
defending $a$' for some $A_1 \subseteq A$ 
and some $a \in A$ is: 
$\lambda \in \Lambda$ 
be such that $A_1 \subseteq \domit(\lambda) \subseteq A$ 
and that if $\lambda(a_1) = \inL$ for every $a_1 \in A_1$, 
then for any $a_x \in A$ 
with $(a_x, a) \in R$, 
there exists some $a_2 \in A_1$ 
such that $(a_2, a_x) \in R$. More informally, regard any subset of arguments 
in $A$ assigned $\inL$ by $\lambda$ as belong to the same group, 
it characterises the collective 
defence of $a$ by the group $A_1$: if 
$a$ is
attacked by an argument, there is at least one member in the group $A_1$ 
that offers a counter-attack to the attacker.   
Any labelling 
$\lambda$ with $\domit(\lambda) = A$ satisfying both conflict-freeness  
and the `$\domit(\lambda)$ 
defending $a$' for every $a \in \domit(\lambda)$ 
is called admissible.  
Assume some label $l_c \in L$ 
as the complement of $\inL$ (`not accept'), and 
for a short-hand, 
assume `$\lambda[a_1, \ldots, a_n] = 
(l_1, \ldots, l_n)$' means both: 
$\domit(\lambda) = \{a_1, \ldots, a_n\}$; 
and  
$\lambda(a_i) = l_i$ for every $1 \leq i \leq n$.  
Then in our example above, 
$\lambda[a_G, a_F, a_E, a_D, a_C, a_B, a_A]$  
can be any one of the following; see (the table referred to by) 
$\Table_1$ in Fig. \ref{fig1}: 
\begin{multicols}{4} 
\begin{itemize} 
	\item $(l_c, \inL, l_c, l_c, l_c, l_c, l_c)$. 
	\item $(l_c, \inL, \inL, l_c, l_c, l_c, l_c)$. 
	\item $(l_c, \inL, l_c, \inL, l_c, l_c, l_c)$. 
	\item $(l_c, \inL, l_c, \inL, l_c, \inL, l_c)$. 
\end{itemize}  
\end{multicols} 

\begin{figure}[!t]  
	\scalebox{0.86}{ 
         \begin{tabular}{c c | c c c c c c c} 
             $\Table_1$ & & $a_G$ & $a_F$ & $a_E$ & $a_D$ & 
		 $a_C$ & $a_B$ & $a_A$\\\hline
		 $\lambda_1$ & & $l_c$  
		 & $\inL$ & $l_c$ &  $l_c$ & $l_c$ & $l_c$ & $l_c$\\  
		$\lambda_2$ & & $l_c$ & \inL & \inL & $l_c$ & $l_c$ & $l_c$ 
		 & $l_c$\\
		 $\lambda_3$ & & $l_c$ & \inL & $l_c$ & \inL & $l_c$ & $l_c$ 
		 & $l_c$\\
		 $\lambda_4$ & & $l_c$ & \inL & $l_c$ & \inL & $l_c$ & 
		 \inL & $l_c$
	 \end{tabular}  
	 {\ }\qquad  
  	\begin{tabular}{c c | c c c c c c c} 
             $\Table_2$ & & $a_G$ & $a_F$ & $a_E$ & $a_D$ & 
		 $a_C$ & $a_B$ & $a_A$\\\hline
		 $\lambda_5$ & & \outL 
		 & \inL & \undecL &  \undecL & \undecL & \undecL & \undecL\\  
		$\lambda_6$ & & \outL & \inL & \inL & \outL & \undecL & \undecL 
		 & \undecL\\
		 $\lambda_7$ & & \outL & \inL & \outL & \inL & \outL & \inL 
		 & \outL 
	 \end{tabular}  
	 }
	 {\ }\\\\\\ 
	 \scalebox{0.88}{
    	\begin{tabular}{c c | c c c c c c c c} 
		$\Table_3$ & & $a_E$ & $a_D$ & $a_C$ & $a_B$ & $a_A$\\\hline
		$\lambda_8$ & & $l_c$ &  \inL & $l_c$ & \inL & $l_c$ 
	 \end{tabular}  
	 {\ }\qquad 
         \begin{tabular}{c c | c c c} 
		 $\Table_4$ & & $a_G$ & $a_F$ & $a_E$\\\hline 
		$\lambda_9$ && \outL & \inL & \undecL \\
		 $\lambda_{10}$ && \outL & \inL & \inL\\ 
		 $\lambda_{11}$ && \outL & \inL & \outL 
	 \end{tabular} 
	 {\ }\qquad 
	 \begin{tabular}{c c | c c c c}  
		 $\Table_5$ & & $a_D$ & $a_c$ & $a_B$ & $a_A$ \\\hline 
		 $\lambda_{12}$ && \inL & \outL & \inL & \outL\\
	 \end{tabular}  
	}
		\caption{(The table referred to by) $\Table_1$ 
		lists admissible labellings ($\lambda_1(a_G)
		= l_c$, and so on). 
		$\Table_2$ lists complete labellings with 
		$\inL, \outL, \undecL$. 
		$\Table_3$ lists labellings that explain $a_B$. 
		$\Table_4$ lists agent 1's local complete labellings 
		for $a_G, a_F$ and $a_E$ based on the knowledge 
		of $(\{a_G, a_F, a_E, a_D\}, 
		\{(a_F, a_G), (a_E, a_D), (a_D, a_E)\})$. 
		$\Table_5$ lists agent 2's local stable labellings 
		for $a_D, a_C, a_B$ and $a_A$ based on 
		the knowledge of 
		$(\{a_D, a_C, a_B, a_A\}, \{(a_E, a_D), (a_D, a_E), 
		(a_D, a_C), (a_C, a_B), (a_B, a_A), (a_A, a_C)\})$.} 
		\label{fig1}
\end{figure}

\noindent {\it Classic argumentation semantics.}  Any admissible labelling 
that also satisfies the `completeness' constraint:
$\lambda$ be such that, for any $a_x \in A$, 
if $\{a \in A \mid \lambda(a) = \inL\}$ defends $a_x$, 
then $a_x \in \domit(\lambda)$ and $\lambda(a_x) = \inL$, is called complete 
in a weaker sense (extension-based approach). 
As proved in \cite{Caminada06}, 
a stronger sense is possible with $\outL$ (`reject') 
or else $\undecL$ (`undecided') following the three rules below: 
given $(A, R)$, for every $a \in A$, \textbf{(Rule 1)} 
	 $\lambda(a) \in \{\inL, \outL, \undecL\}$, 
	 \textbf{(Rule 2)} 
 $\lambda(a) = \inL$ if and only if, or iff, 
	  there is no attacker of $a$ that is not assigned $\outL$ 
	  by $\lambda$, and  
  \textbf{(Rule 3)} $\lambda(a) = \outL$ iff 
	   there is some attacker of $a$ that is assigned $\inL$ 
	   by $\lambda$.  

Labellings obtained by these are conservative over the weaker-sense labellings for $\inL$-labelled 
arguments, but give more information to the acceptance statuses 
of the other arguments. 
Set $L = \{\inL, \outL, \undecL, l_c\}$ 
to express either of the senses. For the stronger one, 
as in $\Table_2$ in Fig. \ref{fig1}, 
$\lambda[a_G, a_F, a_E, a_D, a_C, a_B, a_A]$ 
can be: 
$(\outL, \inL, \undecL, \undecL, \undecL, \undecL, \undecL)$; 
$(\outL, \inL, \inL, \outL, \undecL, \undecL, \undecL)$;  
or $(\outL, \inL, \outL, \inL, \outL, \inL, \outL)$. 

Let $\Lambda^{com}_{(A, R)}$ be the set of all complete labellings 
of $(A, R)$ in either of the senses, and assume 
$\lambda_1 \preceq \lambda_2$ holds 
just when 
firstly $\domit(\lambda_1) = \domit(\lambda_2)$ holds, 
secondly $\bigcup_{a \in \domit(\lambda_1)}\lambda_1(a) = 
\bigcup_{a \in \domit(\lambda_2)}\lambda_2(a)$ holds, and  
thirdly $\lambda_1(a) \not\in \{\undecL, l_c\}$ materially implies 
$\lambda_1(a) = \lambda_2(a)$. Then 
$\lambda$ is called: preferred 
just when firstly it is complete and secondly, 
for any $\lambda_x \in \Lambda^{com}_{(A, R)}$, 
either $\lambda_x \preceq \lambda$ holds or else 
$\lambda_x$ and $\lambda$ are not comparable  
in $(\Lambda^{com}_{(A, R)}, \preceq)$; 
stable just when it is preferred and, 
for any $a_x \in A$, if 
$\lambda(a_x) \not= \inL$, 
then there is some $a \in A$ such that 
$\lambda(a) = \inL$ and that $(a, a_x) \in R$; 
and grounded iff $\lambda$ is the greatest lower bound of 
$(\Lambda^{com}_{(A, R)}, \preceq)$. Classic 
argumentation semantics make use of them:  
	$X$ semantics of $(A, R)$ is the 
		set of all $X$ labellings of $(A, R)$, 
		where $X$ is a member of $\{$Complete, Preferred, 
		Stable, Grounded$\}$.  
Intuition for each semantics is: complete semantics 
is inclusive of all arguments that can be defended; preferred 
semantics is prejudiced for a maximally opinionated complete labelling
by $\preceq$,
grounded semantics is an impartial 
semantics preferring to judge `accept' 
only those arguments that are judged `accept' by all 
complete labellings; and stable semantics ensures 
that the `accept' arguments attack 
every other arguments. \\\\
\noindent {\it Explanation semantics.}   
The defence constraint enforces certain 
causality in the labels of arguments; 
that, given $\lambda$, some argument assuming $\inL$ 
is due specifically to other arguments assuming certain labels. 
Let us look back at our example, and consider 
one admissible labelling $\lambda[a_G, \ldots, a_A] 
= (l_c, \inL, l_c, \inL, l_c, \inL, l_c)$. 
According to the reasoning 
in \cite{Fan15b}, $\lambda(a_B) = \inL$ is explained
by $a_D$'s $\inL$, since, if $a_D$ is not assigned $\inL$, 
$a_B$ also would not be. On the other hand, it is not explained by 
$a_F$'s $\inL$, since the label of $a_F$ 
exerts no influence over the label of $a_B$. Following 
the intuition, we may enforce the constraint of `relevance to $a$' 
for some $a \in A$: 
given $(A, R)$, let $(A', R')$ 
be such that $A' = \bigcup_{(a_x, a) \in R^*} a_x$ 
and that $R' \in \mathcal{R}^{A'}$ ($R^*$ is the reflexive-transitive 
closure of $R$), $\lambda$ shall be such that   
$\domit(\lambda) = A'$. 

Then, for any $a \in A$, 
$\lambda$ explains $a$ iff: firstly $\lambda$ satisfies 
relevance to $a$; 
secondly, given $AF \equiv (\domit(\lambda), R)$ with $R \in \mathcal{R}^{\domit(\lambda)}$,
$\lambda$ is admissible in $AF$; 
and finally $\lambda(a) = \inL$. We may say that 
the set of all labellings that explain $a$ 
is the explanation semantics of $(A, R)$ for $a$. 
In our example, the explanation semantics for $a_B$ 
is a singleton set containing just $\lambda[a_E, \ldots, a_A] 
= (l_c, \inL, l_c, \inL, l_c)$; see $\Table_3$ in Fig. \ref{fig1}. \\

\noindent {\it  Multi-agent semantics.} 
Of many a number of existing multi-agent argumentation semantics, 
there are those that focus on 
derivation of a global argumentation semantics 
from agents' local argumentation semantics. 
One typical idea that we touch upon in this paper is: given $(A, R)$, 
partition  $A$ 
into $A_1, \ldots, A_n$ ($\bigcup_{1 \leq i \leq n}A_i = A$), each of which 
represents arguments put forward by an agent; 
each agent $i$ ($1 \leq i \leq n$) 
knows some $(A_x, R_x)$ such that $A_i \subseteq A_x \subseteq A$ 
and $R_x = R \cap (A_x \times A_x)$, and derives some semantics of 
$(A_x, R_x)$ and keep the labels for arguments in $A_i$ which 
form the agent's local semantics $\Lambda_i$ where 
$\lambda_i \in \Lambda_i$ satisfies $\domit(\lambda_i) = A_i$ 
(a concrete example follows below). 
In the meantime, an external observer computes 
a semantics of the global argumentation graph 
$\Lambda_g$ (where $\lambda_g \in \Lambda_g$ satisfies $\domit(\lambda_g) = A$).
A multi-agent labelling is then any $\lambda_g \in \Lambda_g$ 
such that, for each partition $A_i$,
it assigns the same labels to $A_i$ 
as some $\lambda_i \in \Lambda_i$. All multi-agent labellings 
form a multi-agent semantics. 
In our example, suppose agent 1: puts forward $a_G, \ldots, a_E$; and 
knows $a_G, \ldots, a_D$ and attacks among them. 
Suppose agent 2: puts forward $a_D, \ldots a_A$; and 
knows $a_E, \ldots, a_A$ and attacks among them. 
Suppose agent 1 (resp. agent 2) uses complete semantics (stable semantics) 
with 
$\inL$, $\outL$ and $\undecL$. Then agent 1's (resp. agent 2's) local complete 
(resp. stable) semantics comprises 3 (resp. 1) labellings as shown in $\Table_4$ (resp. 
$\Table_5$) in Fig. \ref{fig1}. Suppose an observer uses 
complete semantics ($\Table_2$). Then the multi-agent semantics is 
a singleton set $\{\lambda_7\}$. 

\section{Relational Argumentation Semantics}   
To understand the different types of argumentation semantics
in a single perspective, we present the viewpoint: 
`argumentation semantics as a relational database' in this section, 
consolidating the bonding and relational generalisation 
of argumentation semantics. We begin by  
formulating  
a tuple relational calculus for argumentation semantics.

\begin{definition}[Arg-labelling table]
	We define an arg-labelling table to be a tuple $(A, \Lambda_1)$  
	with $A \subseteq_{\text{fin}} \mathcal{A}$, 
	called header of $(A, \Lambda_1)$, 
	and $\Lambda_1 \subseteq \Lambda$ with any of its member 
	$\lambda$ satisfying $\domit(\lambda) = A$ 
	called body of $(A, \Lambda_1)$.   
	
\end{definition} 
\begin{example}[Arg-labelling table] 
	We saw some examples of 
	arg-labelling tables, literally expressed 
	in tables, in Fig. \ref{fig1}. 
	Putting the table referred to by $\Table_1$ 
	in tuple form, 
	we obtain $(\{a_G, \ldots, a_A\}, 
	 \{\lambda_1, \ldots, \lambda_4\})$ 
	 where $\lambda_i$, $1 \leq i \leq 4$, is such that 
	 $\domit(\lambda_i) = \{a_G, \ldots, a_A\}$,
	 and that $\lambda_i(a_j) = x_{(i, j)}$, $j \in \{G, \ldots, A\}$ 
	 $(x_{(i, j)}$ 
	 is the element in the table in the $i$-th row and $j$-th column, 
	 as usual of a matrix notation). \hfill$\clubsuit$ 
\end{example} 
It is obvious from the definition that there is no duplicate in the body. Now:  
\begin{definition}[Relational arg-labelling database]
	Let $\mathcal{T}$ be a potentially uncountable set of table names, 
	let $\Table$ with(out) a subscript 
	and a superscript  denote 
	its member, 
	and let $\TABLEit$ with(out) a subscript 
	its subset. 
       We define a relational arg-labelling database schema 
	to be a tuple $(L, \TABLEit, \header)$ 
	with $\header: \TABLEit \rightarrow 
	2^{A}$ associating a set of arguments to each table name. 
	We denote the set of all 
	relational arg-labelling database schemata  
	by $\SCHEMAit$, and refer to its member 
	by $\Schemait$ with(out) a subscript. 

	We define a relational arg-labelling database 
	for $\Schema \equiv (L, \TABLE , \header)$ 
	to be 
	$db: \TABLE \rightarrow 2^{\Lambda}$ satisfying 
	the following condition\footnote{Readers are cautioned 
	that this superscript is $\Lambda$ and not $A$.}:
	for every $\Table \in \TABLE$ and every 
	$\lambda \in db(\Table)$, it holds that 
	$\domit(\lambda) = \header(\Table)$.  
\end{definition} 
\begin{example}[Relational arg-labelling database]\label{exrelationalarglabellingdatabase}    
	Let us refer to Fig. \ref{fig1} 
	and observe that there are 5 arg-labelling tables 
	in total. The arg-labelling database holding 
	them is expressed with the following 
	relational arg-labelling database schema and 
	relational arg-labelling database with respect to it.  
	Assume $A_1 = A_2 = \{a_G, \ldots, a_A\}$, 
	$A_3 = \{a_E, \ldots, a_A\}$, 
	$A_4 = \{a_G, a_F, a_E\}$, 
	$A_5 = \{a_D, \ldots, a_A\}$, 
	 $\Lambda_1 = 
	\{\lambda_1, \ldots, \lambda_4\}$, 
	$\Lambda_2 = \{\lambda_5,  \lambda_7\}$, 
	$\Lambda_3 = \{\lambda_8\}$, 
	$\Lambda_4 = \{\lambda_9, \lambda_{10}, \lambda_{11}\}$, 
	and $\Lambda_5 = \{\lambda_{12}\}$. 

	Then:  
	        $\Schema$ is $(L, 
		       \{\Tableit_1, \ldots, \Tableit_5\},  
		       \header \equiv \{\Tableit_1 \mapsto A_1, \ldots, 
		       \Tableit_5 \mapsto A_5\})$; 
		       and 
	        $db$ is $\{\Tableit_1 \mapsto \Lambda_1, 
		       \ldots,\linebreak \Tableit_5 \mapsto \Lambda_5\}$. 
	It holds that $\domit(\lambda_1) = A_1 = 
	\header(\Tableit_1)$, and similarly 
	for all the others. (Note: the reason 
	that we chose to have all the arg-labelling 
	tables 
	in the same - that is, in one - arg-labelling database is 
	because we focus more on illustration of our 
	relational calculus in sections 3.1 and 3.2.  
	In practice, 
	which tables should be grouped together varies by one's intention
	(e.g. Example 8 later).) 
	\hfill$\clubsuit$ 
\end{example} 

\subsection{Syntax and semantics of formal query language} 
Syntax of the query language is defined almost
as tuple relational calculus, save 
we introduce 
a dyadic function $\countit$ to count the number 
of a certain label assigned to the header arguments by a labelling. 
This additional function is appropriate 
for an arg-labelling table where any member in the body 
is a labelling, that is, the output of any member is in the same type 
$L$, which assures well-definedness of the counting function 
for every arg-labelling table. We introduce definitions (Definition 3 - 5) and illustrate 
them in an example. 

\begin{definition}[Atomic query formula] 
	Let $V$ be an uncountable set 
	of variables. 
	We define 
	an atomic query formula with respect to 
	 $\Schema \equiv (L, \TABLEit, \header)$ 
	 to be any of the following, so long as it satisfies 
	 the accompanying conditions.   
	 \begin{multicols}{2} 
		 \begin{itemize}[leftmargin=0.3cm]
		 \item $v_1.a_1 \doteq v_2.a_2$ for 
			 $v_1, v_2 \in V$ and 
			 $a_1, a_2 \in \mathcal{A}$. 
		 \item $v_1.a_1 \doteq l\quad \ \ \ $ for 
			 $v_1 \in V$, 
			 $a_1 \in \mathcal{A}$ and 
			 $l \in L$. 
		 \item $\Tableit[v_1] \quad\ \ \ \ \ \ \ $ for 
			 $v_1 \in V$ and 
			 $\Tableit \in \TABLEit$.    
		 \item $\countit(v_1, l_1)\ \dot{\leq}\  
			 \countit(v_2, l_2)$ 
			  for 
			 $l_1, l_2 \in L$. 
		 \item $\countit(v_1, l_1)\ \dot{\leq}\  n\ \ \qquad\qquad $ 
			 for  $l_1 \in L$ and $n \in \mathbb{N}$. 
		 \item $n\ \dot{\leq}\ \countit(v_1, l_1)\ \qquad\qquad$ for $l_1 \in L$ 
			 and $n \in \mathbb{N}$.
	 \end{itemize}  
	 \end{multicols} 
\end{definition} 
\begin{definition}[Query formula]   
	We define a query formula with respect to $\Schemait$ to be any of the following. 
	We may refer to  
	a query formula by $\pmb{F}$ with(out) 
	a subscript.
	\begin{multicols}{2} 
		\begin{itemize}[leftmargin=0.3cm]  
		\item an atomic formula with respect to $\Schemait$. 
		\item $\neg \pmb{F}_1$ if $\pmb{F}_1$ is a formula. 
		\item $\pmb{F}_1 \wedge \pmb{F}_2$ if 
			$\pmb{F}_1$ and $\pmb{F}_2$ are formulas. 
		\item $\pmb{F}_1 \vee \pmb{F}_2$ if 
			$\pmb{F}_1$ and $\pmb{F}_2$ are formulas. 
		\item $\exists v:A[\pmb{F}_1]$ if 
			$A$ is a set of arguments and $\pmb{F}_1$ 
			is a formula. 
		\item $\forall v:A[\pmb{F}_1]$ if 
			$A$ is a set of arguments and $\pmb{F}_1$ 
			is a formula. 
	\end{itemize} 
	\end{multicols} 
	Following the predicate logic convention, we 
	say that a variable $v$ is: free  
	in $\pmb{F}$ iff it is not quantified; bound 
	in $\pmb{F}$ iff it is not free in $\pmb{F}$. 
\end{definition} 
The following semantics of the language is 
fairly standard to the database theory community; 
argumentation people may find the example to follow 
instrumental. 
\begin{definition}[Semantics] 
	Let $\evalit: V \rightarrow \Lambda$ 
	be an interpretation function 
        such that $\evalit(v) \in \Lambda$,
	and let a `semantic structure' be a tuple 
	$(\Schema, db, \evalit)$ for $\Schema$ 
	and $db$ for $\Schema$.  
	We inductively define $(\Schema, db, \evalit) \models \pmb{F}$ for some 
	$(\Schema, db, \evalit)$ and some  
	$\pmb{F}$ as follows. 
	{\small 
	\begin{itemize}[leftmargin=0.3cm]  
		\item $(\Schema, db, \evalit) \models  
			v_1.a_1 \doteq v_2.a_2$ iff   
			$a_1 \in \domit(\evalit(v_1))$ 
			$\andC$ $a_2 \in \domit(\evalit(v_2))$ $\andC$ 
			$\evalit(v_1)(a_1) = \evalit(v_2)(a_2)$.  
		\item $(\Schema, db, \evalit) \models v_1.a_1 \doteq l$ 
			{\ }\quad\ \ iff  $a_1 \in \domit(\evalit(v_1))$ 
			$\andC$ $\evalit(v_1)(a_1) = l$.  
		\item $(\Schema, db, \evalit) \models \Tableit[v_1]$ 
			{\ }\quad\ \ \ \ iff $\evalit(v_1) \in db(\Tableit)$.  
		\item {\small $(\Schema, db, \evalit) \models 
			\countit(v_1, l_1)\ \dot{\leq}\ \countit(v_2, l_2)$} 
			iff {\small $|\{a \in \domit(\evalit(v_1)) \mid \evalit(v_1)(a) = l_1\}| \leq |\{a \in \domit(\evalit(v_2)) \mid \evalit(v_2)(a) = l_2\}|$} 
		\item {\small $(\Schema, db, \evalit) \models \countit(v_1, l_1) \ 
			\dot{\leq}\ n$} iff 
			{\small $|\{a \in \domit(\evalit(v_1)) \mid \evalit(v_1)(a) = l_1\}| \leq n$.} 
		\item  {\small $(\Schema, db, \evalit) \models n\ \dot{\leq}\ 
			\countit(v_1, l_1)$} iff 
			{\small $n \leq 
			|\{a \in \domit(\evalit(v_1)) \mid 
			\evalit(v_1)(a) = l_1\}|$}. 
		\item  $(\Schema, db, \evalit) \models \neg \pmb{F}$ 
			{\ }\quad\quad\quad\ \ \ \ \ \ \ \  iff  
			  $(\Schema, db, \evalit) \not\models \pmb{F}$. 
		  \item $(\Schema, db, \evalit) \models \pmb{F}_1 \wedge \pmb{F}_2$
			  {\ }\quad\ \ \ \ \quad  iff
			  $(\Schema, db, \evalit) \models \pmb{F}_i$ 
			  for every $i \in \{1,2\}$. 
		  \item  $(\Schema, db, \evalit) \models \pmb{F}_1 \vee \pmb{F}_2$
			{\ }\quad\ \ \ \quad iff  
			   $(\Schema, db, \evalit) \models \pmb{F}_i$ 
			  for some $i \in \{1,2\}$. 
		  \item $(\Schema, db, \evalit) \models \exists v:A[\pmb{F}]$
			  {\ }\ \ \ \quad \  iff  
			  there is some $\lambda$ such that 
			  $\domit(\lambda) = A$ and that
			  $(\Schema, db, \evalit') \models \pmb{F}$ 
			  where $\evalit'$ is almost 
			  exactly $\evalit$ except  
			  $\evalit'(v) = \lambda$.  
		  \item $(\Schema, db, \evalit) \models \forall v:A[\pmb{F}]$
			  {\ }\quad\quad iff, 
			  for every $\lambda$, 
			  if $\domit(\lambda) = A$, then 
			  $(\Schema, db, \evalit') \models \pmb{F}$ 
			  where $\evalit'$ is almost exactly 
			  $\evalit$ except $\evalit'(v) = \lambda$.
	\end{itemize} 
	}
	We say that $(\Schema, db, \evalit)$ models 
	$\pmb{F}$ iff $(\Schema, db, \evalit) \models \pmb{F}$. 
\end{definition}  
For atomic formulas, 
$v_1.a_1 \doteq v_2.a_2$ 
			tests whether
			the label of $a_1$ 
		        assigned by $\lambda_x \equiv \evalit(v_1)$ 
			and that of $a_2$ assigned by $\lambda_y 
			\equiv \evalit(v_2)$ 
			matches. The first two attached conditions 
			force 
			$a_1 \in \domit(\lambda_x)$ 
			and $a_2 \in \domit(\lambda_y)$.   
			$v_1.a_2 \doteq l$ tests 
			whether the label of $a_1$  
			assigned by $\lambda_x \equiv \evalit(v_1)$ 
			is $l$. $\Tableit[v_1]$ tests 
			whether $\lambda_x \equiv \evalit(v_1)$ 
			is in the body of the arg-labelling table 
			$db(\Tableit)$. 
			$\countit(v_1, l_1)\ \dot{\leq} \ 
			\countit(v_2, l_2)$ compares 
			the number of arguments assigned $l_1$ 
			by $\lambda_1 \equiv \evalit(v_1)$ 
			and that of arguments assigned $l_2$ 
			by $\lambda_2 \equiv \evalit(v_2)$, 
			and similarly for the other two 
			atomic formulas involving 
			$\countit$. 
                        
\begin{example}[Formal language of queries]
	We assume the arg-labelling tables in our running example 
	(see Fig. \ref{fig1}). 
	\begin{itemize} 
		\item $\Tableit_1[v_1] \wedge \Tableit_2[v_2] 
			\wedge v_1.a_{F} \doteq v_2.a_B$ 
			signifies firstly that 
			$\lambda_x \equiv \evalit(v_1)$  
			is in the arg-labelling table $\Tableit_1$, 
			secondly that $\lambda_y \equiv \evalit(v_2)$ 
			is in the arg-labelling table $\Tableit_2$, 
			and thirdly that 
			$\lambda_x(a_{F}) = \lambda_y(a_B)$.  
			Compare $\Tableit_1$ and $\Tableit_2$; 
			for any labelling $\lambda$ in $\Tableit_1$, 
			$\lambda(a_F) = \inL$, 
			and for any labelling $\lambda'$ in $\Tableit_2$, 
			$\lambda'(a_k) = \inL$ iff 
			$\lambda' = \lambda_7$. This means 
			that, for any $\evalit$, 
			$(\Schemait, db, \evalit)$ models 
			this query formula 
			iff ${(\lambda_x =)}\ \evalit(v_1) \in \{\lambda_1, \ldots, 
			\lambda_4\}$ $\andC$ 
			${(\lambda_y =)}\ \evalit(v_2) = \lambda_7$.   
		\item  $\Tableit_4[v_2]  
			\wedge \exists v_1:\{a_{G},\ldots, a_A\}[\Tableit_2[v_1] 
			\wedge v_1.a_{E} \doteq 
			v_1.a_{A} \wedge 
			 v_2.a_{E} \doteq v_1.a_F]$ signifies 
			the following: for 
			some $\lambda_x \equiv \evalit(v_1)$ in the body 
			of $\Table_2$ with $\lambda_x(a_{E}) = 
			\lambda_x(a_{A})$,   
			$\lambda_y \equiv \evalit(v_2)$ 
			is such that firstly 
			$\lambda_y$ is in the body of 
			$\Table_4$ and 
			secondly $\lambda_y(a_E) = \lambda_x(a_F) (= \inL)$. 
			There is only one labelling 
			in $\Table_4$, $\lambda_{10}$, to be 
			the $\lambda_y$. 
			This means that, for any $\evalit$, 
			$(\Schemait, db, \evalit)$ models 
			this query formula iff ${(\lambda_y =)}\ 
			\evalit(v_2) = \lambda_{10}$. 
		\item $\Tableit_2[v_1] \wedge (3 \ \dot{\leq}\ 
			\countit(v, \inL) \vee 4 \ \dot{\leq}\ \countit(v, \undecL))$ signifies firstly that $\lambda_x \equiv \evalit(v_1)$ 
			is in the arg-labelling table $\Tableit_2$, 
			and secondly that 
			$\lambda_x$ assigns 
			 $\inL$ to at least 3 or $\undecL$ to at least 
			 4 arguments in 
			 $\domit(\lambda_x)$. 
			 This means that, for any $\evalit$, 
			 $(\Schemait, db, \evalit)$ models 
			 this query formula iff 
			 $\evalit(v) \in \{\lambda_5, \lambda_7\}$.
	\end{itemize}  
	In the above examples, take particular notice of the purpose of 
	$\Tableit[v]$ atomic formulas acting to restrict 
	the domain of a labelling to the header arguments 
	of the referred arg-labelling table. \hfill$\clubsuit$
\end{example} 
\subsection{Relational arg-labelling semantics} 
\begin{definition}[Relational arg-labelling semantics] 
      We define an arg-labelling query 
	to be 
	the following expression:  
	$\{v : A \mid \pmb{F}\}$, 
	whereby the only free variable in $\pmb{F}$ is $v$.   
      
        We define the semantics of 
	a query $\{v:A \mid \pmb{F}\}$ 
	with respect to some $\Schemait$ and $db$ to be 
	the set of  
	all $\lambda$ satisfying 
	$(\Schemait, db, \evalit) \models 
	\pmb{F}$ with $\evalit(v) = \lambda$, 
	and denote it by $\norm{\{v:A \mid \pmb{F}\}}$. 
	We call any $\norm{\{v:A \mid \pmb{F}\}}$ with respect to 
	$\Schemait$ and $db$ 
	a relational arg-labelling semantics with respect to 
	$\Schemait$ and $db$. 
\end{definition} 
Basic database queries such as selecting rows 
of a table, selecting columns of a table, 
and joining two tables, 
produce certain effects on 
arg-labelling database with associated 
arg-labelling tables.   

\begin{example}[Selecting columns of a single arg-labelling table]\label{exselectioncolumn} 
	In Fig. \ref{fig1}, 
	$\Tableit_4$ (more precisely 
	the arg-labelling table referred to by $\Tableit_4$) 
	is the result of selecting 3 columns $a_{G}$, $a_F$ 
	and $a_E$ 
	of $\Tableit_2$. 
	The corresponding relational arg-labelling 
	semantics is $\norm{\{v: \{a_{G}, a_{F}, a_E\} \mid 
	\Tableit_2[v]\}}$ with respect to 
	the same $\Schema$ and $db$ in Example 
	\ref{exrelationalarglabellingdatabase}.
To see to it, 
	we firstly enumerate all $\lambda$ 
	with $(\Schemait, db, \evalit) \models \Tableit_2[v]$ 
	for $\evalit(v) = \lambda$, 
	which are $\lambda_1, \ldots,$ and $\lambda_4$;
	for each of them, we force the domain to $\{a_G, a_F, a_E\}$, 
	to obtain $\lambda_9, \lambda_{10}$ and $\lambda_{11}$. 
	We saw in Section 2 that both explanation 
	and multi-agent semantics share the process 
	of restricting attention to a subset of arguments.
	\hfill$\clubsuit$ 
\end{example}  
In the remaining, we let $\bigwedge_{1 \leq i \leq n} 
	v.a_i \doteq l_i$ abbreviate 
	$v.a_1 \doteq l_1 \wedge \ldots \wedge 
	v.a_n \doteq l_n$.  
\begin{example}[Selecting rows of a single arg-labelling table]\label{exselectingrows}
	Classic argumentation semantics is defined with row selection 
	from complete semantics. With $\Table_2$, 
	{\small $\norm{\{v: \{a_G, \ldots, a_A\} \mid  \Tableit_2[v]
	 \wedge 
	\neg 1 \dot{\leq} \countit(v, \undecL)\}}$} with respect to the same 
	$\Schemait$ and $db$ is the stable semantics. 
	 \hfill$\clubsuit$ 
\end{example} 

\noindent To put together multiple tables based on some constraints, 
there are joins. 
\begin{example}[Condition join]\label{exconditionjoin} 
	Let us say, we like to join $\Table_4$ and $\Table_5$ 
	in such a way that $\lambda$ with $\domit(\lambda) = 
	\{a_G, \ldots, a_A\}$ is in the resultant table's body 
	iff there are some $\lambda_x \in db(\Table_4)$ 
	and some $\lambda_y \in db(\Table_5)$ 
	such that $\lambda_x(a_i) = \lambda(a_i)$ 
	and $\lambda_y(a_j) = \lambda(a_j)$ hold for 
	every $i \in \{a_G, a_F, a_E\}$ and every 
	$j \in \{a_D, \ldots, a_A\}$ and that 
	$\lambda$ is in the body of $\Table_2$. 
	The relational arg-labelling semantics 
	of this condition join is 
	$\norm{\{v: \{a_G, \ldots, a_A\} \mid \Tableit_2[v] 
	\wedge \exists 
	v_1:\{a_G, a_F, a_E\}[\Tableit_4[v_1] \wedge  
	(\bigwedge_{i \in \{G, \ldots, E\}}v.a_i \doteq v_1.a_i)
	\wedge 
	 \exists v_2:\{a_D, \ldots, a_A\}[\Tableit_5[v_2] \wedge
	 \bigwedge_{i \in \{D, \ldots, A\}}v.a_i \doteq v_2.a_i]]\}}$.  
	 This, incidentally, forms the multi-agent semantics  
	 of Section 2.  
	 \hfill$\clubsuit$ 
\end{example} 

Though we do not highlight here for space, there are 
other types of joins, too, such as 
outer joins, as well as renaming, division and typical 
set-theoretic operations (union, intersection, difference). 
Readers are referred to 
\cite{Ramakrishnan02,Kolaitis09}.  

\subsection{Consequence of the perspective 
of argumentation semantics as a relational database}
We have provided a fresh perspective 
of: traditional argumentation semantics 
as a relational database. 
Since our relational calculus 
is tuple relational calculus plus a specific counting function, 
every sensible query (technically, domain independent queries
\cite{Ramakrishnan02,Kolaitis09}) that is possible within the calculus 
can be queried in SQL that is readily available for use.  

For theoretical implications, it offers a unifying relational 
point of view to 
 argumentation semantics. For a start,  weaker or stronger, whichever 
 sense it is, we can uniformly define classic argumentation 
 semantics. 
\begin{proposition}[Classic argumentation semantics
	as  relational arg-labelling semantics]
	Given $(A, R)$, assume  
	$\Schema \equiv (L, \{\Tableit_{com}\}, 
	\header)$ and $db$ such that 
	$\header(\Tableit_{com}) = A$, and that  
	 $db(\Tableit_{com}) = \Lambda^{com}_{(A, R)}$. 
	Then, 
	complete / preferred / grounded / stable semantics of $(A, R)$ 
	is a relational arg-labelling semantics with respect to 
	$\Schema$ and $db$.
\end{proposition}  
\begin{proof}  
       Vacuous for complete semantics; we need only select 
	all members of the body of $\Table_x$. See Example \ref{exselectingrows} 
	for row selection.  
	For preferred semantics, we need only ensure to 
	drop all rows that are not maximal in 
	\mbox{$(\Lambda^{com}_{(A, R)}, \preceq)$}. 
	The corresponding relational 
	arg-labelling semantics 
	is: $\norm{\{v: A \mid \exists v_1:A[\Tableit_x[v_1] 
	\wedge \bigwedge_{a_i \in A}v.a_i \doteq v_1.a_i]
	\wedge \forall v_2:A[\neg \Tableit_x[v_2] 
	\vee \bigwedge_{i \in A}((\neg v_2.a_i \doteq \inL
	\wedge \neg v_2.a_i \doteq \outL) 
	\vee v_2.a_i \doteq v.a_i)]\}}$. 
	For grounded semantics, we need to derive the meet of   
	$\Lambda^{com}_{(A, R)}$ in $\preceq$, but this is precisely the 
	relational arg-labelling semantics: 
	$\norm{\{v: A \mid  \bigwedge_{a_i \in A}(
	((\neg v.a_i \doteq \inL \wedge \neg v.a_i \doteq \outL) \vee 
	\forall v_1:A[\neg \Tableit_x[v_1] 
	\vee v.a_i \doteq v_1.a_i]) \wedge ((\neg v.a_i \doteq \undecL 
	\wedge \neg v.a_i \doteq l_c) 
	\vee \exists v_1:A[\Table_x[v_1] \wedge (v_1.a_i \doteq \undecL \vee 
	v_1.a_i \doteq l_c \vee 
	\exists v_2:A[\Tableit_x[v_2] \wedge \neg v_1.a_i \doteq v_2.a_i])]))\}}$.   For stable semantics, $\norm{\{v: A \mid 
	\Tableit_x[v] \wedge \neg 1 \dot{\leq} \countit(v, \undecL) 
	\wedge \neg 1 \dot{\leq} \countit(v, l_c)\}}$. 
	\hfill$\Box$ 
\end{proof}   
This serves to underscore how derivation of some 
semantics requires the 
underlying argumentation graph, and how, on the other hand, 
other semantics only require comparisons of labels 
with no further reference to the argumentation graph, 
for which it makes a perfect sense to adopt the relational 
perspective. 

The relational perspective goes beyond reformulation,  
contributing to generalisation. As we saw in Section 2, 
some argumentation semantics (explanation, multi-agent) intend to 
derive a view 
covering only a part of an argumentation graph in order 
to compute the final result. We can understand it as: 
\begin{proposition}[Partial semantics] 
	Given $(A, R)$, 
	assume some semantics of it: $\Lambda^{\sem}_{(A, R)}$ 
	such that $\lambda \in \Lambda^{\sem}_{(A, R)}$ 
	satisfies $\domit(\lambda) = A$. 
	Let us say that $\Lambda_1$ is 
	a partial semantics of $\Lambda^{\sem}_{\FFMMA}$ 
	iff all the following conditions hold. 
	\begin{enumerate} 
		\item  There is some $A_1 \subseteq A$ 
			such that, for every $\lambda_x \in \Lambda_1$,
			$\domit(\lambda_x) = A_1$ holds. 
		\item For every $\lambda_x \in \Lambda_1$ 
	and for every $a_1 \in A_1$,  
	there is some $\lambda \in \Lambda^{\sem}_{\FFMMA}$ 
	such that $\lambda_x(a_1) = \lambda(a_1)$. 
		\item For every $\lambda \in \Lambda^{\sem}_{\FFMMA}$ 
			and for every $a_1 \in A_1$, 
			there is some $\lambda_x \in 
			\Lambda_1$ such that 
			$\lambda_x(a_1) = \lambda(a_1)$.  
	\end{enumerate}   
        
	Every partial semantics of $\Lambda^{\sem}_{\FFMMA}$ 
	is a relational arg-labelling semantics 
	with respect to 
	$\Schema \equiv (L, \{\Tableit_{\sem}\} \cup \TABLEit, \header)$ and $db$ such that 
	$\header(\Tableit_{\sem}) = A$ 
	and $db(\Tableit_{\sem}) = \Lambda^{\sem}_{\FFMMA}$. 
	Assume $\TABLEit$ is some set of table names. 
\end{proposition} 
\begin{proof} 
       The defined partial semantics is selection by column in the relational
	perspective; see \mbox{Example} \ref{exselectioncolumn}. 
	Indeed, assume an arbitrary $A_1 \subseteq A$
	and assume that $\Lambda_1$ 
	is a partial semantics of $\Lambda^{\sem}_{\FFMMA}$. 
	Then $\Lambda_1 = 
	\norm{\{v: A_1 \mid \Tableit_{\sem}[v]\}}$. 
	\hfill$\Box$ 
\end{proof} 
In Fig. \ref{fig1}, $\Tableit_4$ is a partial semantics 
of $\Tableit_2$ which is a complete semantics of the argumentation 
graph of the example, and 
$\Tableit_5$ is a partial semantics 
of stable semantics of the argumentation graph. 
Together with the observation in Ex. \ref{exconditionjoin},
we see that multi-agent semantics is a relational arg-labelling 
semantics.

\begin{proposition}[Dependent semantics] 
     Given $\FFMMA$,  
	assume some semantics  
	of it, say 
	$\Lambda^{\sem}_{\FFMMA}$ where 
	$\domit(\lambda) = A$ 
	for every $\lambda \in \Lambda^{\sem}_{(A, R)}$, 
	and some $\Lambda_x$ where 
	for every $\lambda_x \in \Lambda_x$, 
	there is some $A_1 \subseteq A$ such that 
	$\domit(\lambda_x) = A_1$. 
	Let us say that $\Lambda_y$ is 
	a dependent semantics of  
	$\Lambda^{\sem}_{\FFMMA}$
	 with respect to  $\Lambda_x$ 
		iff $\Lambda_y \subseteq \Lambda^{\sem}_{(A, R)}$ holds 
		and also, for any $\lambda_y \in \Lambda_y$, 
	there is some $\lambda_x \in \Lambda_x$
			such that $\lambda_x(a_1) = \lambda_y(a_1)$ 
			for every $a_1 \in A_1$. 
	
	Every dependent semantics of $\Lambda^{\sem}_{\FFMMA}$ with 
	respect to 
	$\Lambda_x$  is a relational 
	arg-labelling semantics with respect to 
	$\Schema \equiv (L, \{\Tableit, \Tableit_x\} \cup \TABLEit, \header)$ 
	and $db$ with: $\header(\Tableit) = A$; 
	$\header(\Tableit_x) = A_1$; 
	$db(\Tableit) = \Lambda^{\sem}_{(A, R)}$; and 
	$db(\Tableit_x) = \Lambda_x$. Assume $\TABLEit$ is some 
	set of table names. 
\end{proposition} 
\begin{proof}   
      Do the described selection of labellings 
	in $\Lambda^{\sem}_{(A,R)}$. 
	\hfill$\Box$ 
\end{proof}  
\begin{example}[Explanation semantics \cite{Fan15b}]  
        Let us look back at 
	Fig. \ref{fig1}. 
	Assume $\Schemait \equiv (L, \{\Tableit_1\}, \header)$ 
	with $\header(\Tableit_1) = \{a_G, \ldots, a_A\}$ 
	and $db \equiv \{\Tableit_1 \mapsto \{\lambda_1,\ldots, 
	\lambda_4\}\}$. 
	Recall $\Tableit_1$ 
	lists all the admissible labellings 
	with $\inL$ and $l_c$. For explanation semantics 
	for $a_B$, firstly only those arguments $a$ with $(a, a_B) \in R^*$ 
	were focused. As we mentioned earlier, 
	we can obtain this as a partial semantics 
	$\Lambda_x \equiv \{\lambda'_1, \ldots, \lambda_4'\}$  
	of $\{\lambda_1, \ldots, \lambda_4\}$ 
	where, for every $1 \leq i \leq 4$
	and for every $a \in \{a_E, \ldots, a_A\}$, 
	$\lambda_i(a) = \lambda_i'(a)$. 
	As the second step, we need to retain only 
	those labelling $\lambda$ in the body of the resultant 
	arg-labelling table with $\lambda(a_B) = \inL$. 
	But this is a dependent semantics 
	of $\Lambda_x$ with respect to 
	$\{a_B \mapsto \inL\}$.
	\hfill$\clubsuit$ 
\end{example} 

\section{Conclusion with Related Work}
Among the existing 
papers, the idea to put an argumentation graph 
into database for queries is around, e.g. 
\cite{Capobianco05,Egly10,Gordon08,Muller14,Pradhan03}; however,
as far as we are aware, 
no previous work proposed the viewpoint: argumentation semantics 
as a relational database. 
We showed that it facilitates encapsulation 
of argumentation semantics in a single relational perspective, 
allowing for natural generalisation of argumentation semantics 
as was demonstrated in section 3.3. Given a plethora 
of argumentation semantics being proposed in various ways, 
there is always a question of how they 
may link in what way. Identification 
of common constraints among different 
argumentation formalisms is popular \cite{Dung95,Baroni18,Amgoud18,Baroni07}. 
The relational perspective, we believe, 
is a fruitful  direction also for the unification research, 
given the following properties. (1) 
	 The tuple relational calculus 
we defined for formal argumentation semantics 
is a proper formal language (is a fragment of predicate logic) 
whose semantics is not only defined at some points (contrarily, 
the constraint identification research identifies a specific set of 
constraints), but for any expression allowed in the language. 
This accommodates an easier fine-tuning of some existing semantics 
to cater to some more specific situations. We saw a couple of examples 
in section 3.3; from partial semantics to multi-agent semantics with 
a join, and from dependent semantics to explanation semantics. 
	(2) Still, the expressiveness of the relational calculus 
		is sufficiently restricted, allowing 
		processing in SQL: as is well-known, 
		any domain independent query 
		\cite{Ramakrishnan02,Kolaitis09} possible 
		in tuple relational calculus can be handled in SQL; 
		and the counting function we incorporated 
		is also processable in SQL. 
	(3) By posing argumentation semantics as 
		a relational database, it becomes more accessible  
		for database theory community to 
		focus on database theory issues in 
		formal argumentation, which 
		we believe will aid furthering communication 
		among relevant research communities for 
		new theoretical and practical insights. 

\nocite{*}
\bibliographystyle{abbrv}
\bibliography{referenceshorter}
\end{document}